\documentclass[10pt,twocolumn,letterpaper]{article}

\usepackage{cvpr}
\usepackage{times}
\usepackage{epsfig}
\usepackage{graphicx}
\usepackage{amsmath}
\usepackage{amssymb}

\usepackage{amsmath}        
\usepackage{amsthm}         
\usepackage{amssymb}        
\usepackage{mathtools}      
\usepackage[ruled,vlined]{algorithm2e}    
\usepackage{algpseudocode}    
\usepackage{graphicx}       
\usepackage{subcaption}     
\usepackage{appendix}       
\usepackage{float}          

\usepackage{multirow}       
\usepackage{lipsum}

\DeclareMathOperator*{\argmax}{argmax}
\DeclareMathOperator*{\argmin}{argmin}

\newtheorem{theorem}{Theorem}[section]

\DeclareMathOperator{\xinput}{\mathbf{x}}
\DeclareMathOperator{\Js}{J_{\text{s}}}
\DeclareMathOperator{\Jt}{J_{\text{t}}}
\DeclareMathOperator{\J}{J}

\DeclarePairedDelimiterX{\infdivx}[2]{(}{)}{%
  #1\;\delimsize\|\;#2%
}
\newcommand{\klinfdiv}{KL\infdivx}

\usepackage[breaklinks=true,bookmarks=false]{hyperref}

\cvprfinalcopy 

\def\cvprPaperID{4226} 

\ifcvprfinal\fi
\begin{document}

\title{What it Thinks is Important is Important: \\Robustness Transfers through Input Gradients}

\author{Alvin Chan$^1$\thanks{Corresponding author: \texttt{guoweial001@ntu.edu.sg}},~~ Yi Tay$^1$,~~ Yew-Soon Ong$^{1,2}$\\
$^1$Nanyang Technological University, \:\:\: $^2$
AI3, A$^{*}$STAR, Singapore
}


\maketitle

\begin{abstract}
Adversarial perturbations are imperceptible changes to input pixels that can change the prediction of deep learning models. Learned weights of models robust to such perturbations are previously found to be transferable across different tasks but this applies only if the model architecture for the source and target tasks is the same. Input gradients characterize how small changes at each input pixel affect the model output. Using only natural images, we show here that training a student model's input gradients to match those of a robust teacher model can gain robustness close to a strong baseline that is robustly trained from scratch. Through experiments in MNIST, CIFAR-10, CIFAR-100 and Tiny-ImageNet, we show that our proposed method, input gradient adversarial matching (IGAM) \footnote{Codes and models are released at: \texttt{https://github.com/alvinchangw/IGAM\_CVPR2020}}, can transfer robustness across different tasks and even across different model architectures. This demonstrates that directly targeting the semantics of input gradients is a feasible way towards adversarial robustness.

\end{abstract}

\section{Introduction}
Deep learning models have shown remarkable performances in a wide range of computer vision tasks \cite{lecun2015deep,touvron2019fixing,lin2019coco} but can be easily fooled by adversarial examples \cite{szegedy2013intriguing}. These examples are crafted by imperceptible perturbations and can manipulate a model's prediction during test time. Due to its potential security risk in deployment of deep neural networks, adversarial examples have received much research attention with many new attacks \cite{carlini2017towards,papernot2018cleverhans,croce2019minimally} and defenses \cite{schott2018towards,prakash2018deflecting,liao2018defense,gowal2018effectiveness,zhang2019theoretically,andriushchenko2019provably} proposed recently.

While there is still a wide gap between accuracy on clean and adversarial samples, the strongest defenses rely mostly on adversarial training (AT) \cite{goodfellow2014explaining,madry2017towards,shafahi2019adversarial}. Adversarial training's main idea, simple yet effective, involves training the model with adversarial samples generated in each training loop. However, crafting strong adversarial training samples is computationally expensive as it entails iterative gradient steps with respect to the loss function \cite{kannan2018adversarial,xie2019feature}.

To circumvent the cost of AT, a recent line of work explores transferring adversarial robustness from robust models to new tasks \cite{hendrycks2019using,shafahi2019adversarially}. To transfer to a target task, current such techniques involve finetuning new layers on top of robust feature extractors that were pre-trained on other domains (source task). While this approach is effective in transferring robustness across different tasks, it assumes that the source task and target task models have similar architecture as pre-trained weights are the medium of transfer.

Here, we propose a robustness transfer method that is both task- and architecture-agnostic with input gradient as the medium of transfer. Our approach, input gradient adversarial matching (IGAM), is inspired by observations \cite{tsipras2018robustness,etmann2019connection} that robust AT-trained models display visibly salient input gradients while their non-robust standard trained models have noisy input gradients (Figure~\ref{fig:input gradients nonrobust and robust}). The value of input gradient at each pixel defines how a small change there can affect the model's output and can be loosely thought as to how important each pixel is for prediction. Here, we show that learning to emulate how robust models view `importance' on images through input gradients can result in robust models even without adversarial training examples.

The core idea behind our approach is to train a student model with an adversarial objective to fool a discriminator into perceiving the student's input gradients as those from a robust teacher model. To transfer across different tasks, the teacher model's logit layer is first briefly finetuned on the target task's data, like in \cite{shafahi2019adversarially}. Subsequently, the teacher model's weights are frozen while a student model is adversarially trained with a separate discriminator network in a min-max game so that the input gradients from the student and teacher models are semantically similar, i.e., indistinguishable for the discriminator model \cite{goodfellow2014generative}.

Through experiments in MNIST, CIFAR-10, CIFAR-100 and Tiny-ImageNet, we show that input gradients are a feasible medium to transfer robustness, outperforming finetuning on transferred weights. Surprisingly, student models even outperform their teacher models in both clean accuracy and adversarial robustness. In some cases, the student model's adversarial robustness is close to that of a strong baseline that is adversarially trained from scratch. Though our method does not beat the state of the art robustness, it shows that addressing the semantics of input gradients is a new promising way towards robustness.

In summary, the key contributions of this paper are as follows:
\begin{itemize}
    \item For the first time, we show that robustness can transfer across different model architectures.
    \item We achieve this by training the student model's input gradients to semantically match those of a robust teacher model through our proposed method.
    \item Through extensive experiments, we show that input gradients are a more effective and versatile medium to transfer robustness than pre-trained weights.
\end{itemize}

\section{Background}
We review the concept of adversarial robustness for image classification and its relationship with input gradients.
\paragraph{Adversarial Robustness}
We express an image classifier as $f(\xinput ; \theta): \xinput \mapsto \mathbb{R}^k$ that maps an input image $\xinput$ to output probabilities for $k$ classes in set $C$, where classifier's parameters is defined as $\theta$. Denoting training dataset as $D$, empirical risk minimization is the standard way to train a classifier $f$, through $\min_{\theta} \mathbb{E}_{(\xinput, \mathbf{y}) \sim D} L(\xinput, \mathbf{y})$, where $\mathbf{y} \in \mathbb{R}^k$ is the one-hot label for the image and $L(\xinput, \mathbf{y})$ is the standard cross-entropy loss: 

\begin{equation}
 L(\xinput, \mathbf{y}) = \mathbb{E}_{(\xinput, \mathbf{y}) \sim D} \left[ - \mathbf{y}^\top \log f (\xinput) \right]
\end{equation}

With this training method, deep learning models typically show good performance on clean test samples but fail in the classification of adversarial test samples. With an adversarial perturbation of magnitude $\varepsilon$ at input $\xinput$, a model is considered robust against this attack if
\begin{equation}
    \argmax_{i \in C} f_i(\xinput ; \theta) = \argmax_{i \in C} f_i(\xinput + \delta ; \theta) 
\end{equation}

where $\forall \delta \in B_p (\varepsilon) = {\delta : \| \delta \|_p \leq \varepsilon}$. With small $\varepsilon$, adversarial perturbation with $p = \infty$ is often imperceptible and is the focus in this paper.

\paragraph{Input Gradients of Robust Models}

Input gradients characterize how an infinitesimally small change to the input affects the output of the model. Given a pair of input and label $(\xinput, \mathbf{y})$, its corresponding input gradient $\nabla_{\xinput} \mathcal{L(\xinput, \mathbf{y})}$ can be computed through gradient backpropagation in a neural network to its input layer. For classification tasks, the input gradient can be loosely interpreted as a pixel map of what the model thinks is important for its class prediction.

It was observed \cite{tsipras2018robustness} that robust models that are adversarially trained display an interesting phenomenon: they produce salient input gradients that loosely resemble input images while less robust standard models display noisier input gradients (Figure~\ref{fig:input gradients nonrobust and robust}). \cite{etmann2019connection} shows in linear models that distance from samples to decision boundary increases as the alignment between the input gradient and input image grows but this weakens for non-linear neural networks. While these previous studies show that robustly trained models result in salient input gradients, our paper studies input gradients as a medium to transfer robustness across different models.

\begin{figure}[ht]
    \centering
    \includegraphics[width=0.6\linewidth]{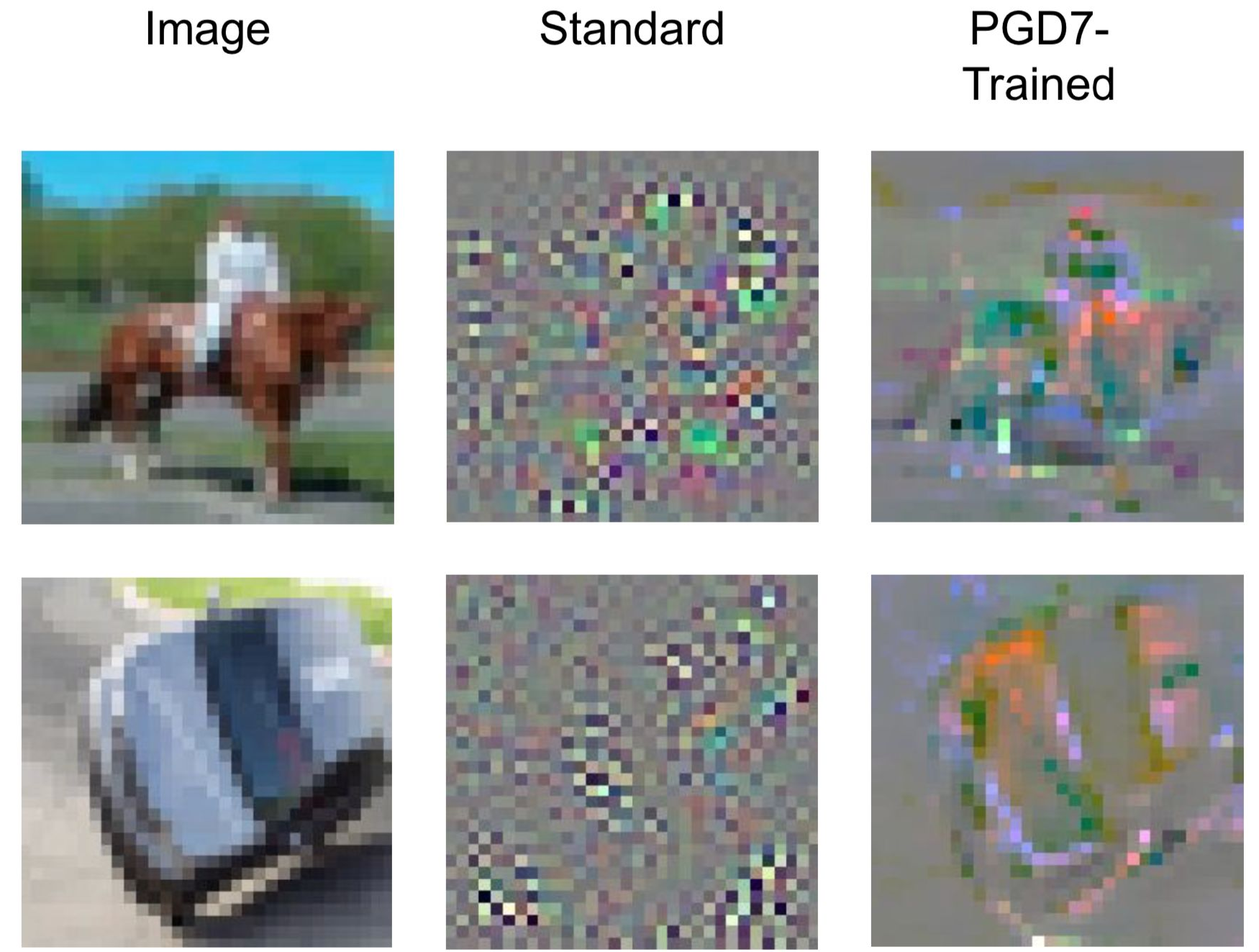}
    \caption{Input gradients of (middle) a non-robust model and (right) robust model on CIFAR-10 images. The non-robust model undergoes standard SGD training with natural images while the robust model is trained with 7-step PGD adversarial examples.}
    \label{fig:input gradients nonrobust and robust}
\end{figure}

\section{Related Work}
We review prior art on defense against adversarial examples and highlight those that are most similar to our work.

\paragraph{Adversarial Training}
With the aim of gaining robustness against adversarial examples, the core idea of adversarial training (AT) is to train models with adversarial training examples. Formally, AT minimizes the loss function:

\begin{equation} \label{eq:AT objective}
L(\xinput, \mathbf{y}) = \mathbb{E}_{(\xinput, \mathbf{y}) \sim D} \left[ \max_{\delta \in B(\varepsilon)} L(\xinput + \delta, \mathbf{y}) \right]
\end{equation}

where $\max_{\delta \in B(\varepsilon)} L(\xinput + \delta, \mathbf{y})$ is computed via gradient-based optimization methods. One of the strongest defenses employ projected gradient descent (PGD) which carries out the following gradient step iteratively:
\begin{equation} \label{eq:pgd step}
    \delta \gets \mathrm{Proj} \left[ \delta - \eta ~ \mathrm{sign} \left( \nabla_{\delta} L(\xinput + \delta, \mathbf{y}) \right) \right ]
\end{equation}
where $\mathrm{Proj}(\xinput) = \argmin_{\epsilon \in B(\varepsilon)} \| \xinput - \epsilon \|$. 

AT has seen many adaptations since its introduction. A recent work \cite{zhang2019defense} seeks to generate more effective adversarial training examples through maximizing feature matching distance between those examples and clean samples. To smoothen the loss landscape so that model prediction is not drastically affected by small perturbations, \cite{qin2019adversarial} proposed minimizing the difference between the linearly estimated and real loss value of adversarial examples. Another work, TRADES \cite{zhang2019theoretically}, reduces the difference between the prediction of natural and adversarial examples through a regularization term to smoothen the model's decision boundary.

\paragraph{Non-Adversarial Training Defense}
Closely linked to our method, there is a line of work that regularizes the input gradients to boost robustness. Those prior art \cite{ross2018improving,jakubovitz2018improving} focus on using double backpropagation \cite{drucker1991double} to minimize the input gradients' Frobenius norm. Those approaches aim to constrain the effect that changes at individual pixels have on the classifier's output but not the overall semantics of the input gradients like our method. \cite{chan2019jacobian} show that models can be more robust when regularized to produce input gradients that resemble input images.

Several recent methods fall under the category of provable defenses that seeks to bound minimum adversarial perturbation for a subset of neural networks \cite{hein2017formal,raghunathan2018semidefinite,wong2018scaling}. These defenses typically first find a theoretical lower bound for the adversarial perturbation and optimize this bound during training to boost adversarial robustness.

\paragraph{Robustness Transfer}
There is a line of work that shows robustness can transfer from one model to another. \cite{hendrycks2019using} shows that robustness from adversarial training can be improved if the models are pre-trained from tasks from other domains. Another work shows that adversarially trained learn robust feature extractors that can be directly transferred to a new task by finetuning a new logit layer on top of these extractors \cite{shafahi2019adversarially}. Circumventing adversarial training, these transferred models can still retain a high degree of robustness across tasks. Unlike our method, these two work require that the source and target models both have the same model architecture since pre-trained weights are directly transferred.

\section{Input Gradient Adversarial Matching}
Our proposed training method consists of two phases: 1) finetuning robust teacher model on target task and 2) adversarial regularization of input gradients during the student models' training.

\subsection{Finetuning Teacher Classifier}
The first stage involves finetuning the weights of the teacher model $f_{\text{t}}$ on the target task. Parameterizing the model weights as $\psi$, the finetuning stage minimizes the cross-entropy loss over the target task training data $(\xinput, \mathbf{y}) \sim \mathcal{D}_{\text{target}}$:

\begin{equation}
    L_{\psi,\text{xent}}(\xinput, \mathbf{y}, \psi) = \mathbb{E}_{(\xinput, \mathbf{y})} \left[ - \mathbf{y}^\top \log f_{\text{t}} (\xinput) \right]
\end{equation}

where $\xinput \in \mathbb{R}^{hwc}$ for $h \times w$-size images with $c$ channels, $\mathbf{y} \in \mathbb{R}^k$ is one-hot label vector of $k$ classes.

To preserve the robust learned representations in the teacher model \cite{shafahi2019adversarially}, we freeze all the weights and replace the final logits layer to finetune. Denoting the frozen weights as $\psi^{\dagger}$ and the new logits layer as $\psi_{\text{logit}}$, the teacher model finetuning objective is

\begin{equation}
    \psi_{\text{logit}}^* = \argmin_{\psi_{\text{logit}}} L_{\text{xent}}( \mathbf{z} (\xinput, \psi^{\dagger}), \mathbf{y}, \psi_{\text{logit}})
\end{equation}

where $\mathbf{z} (\xinput, \psi^{\dagger})$ represents the hidden features before the logit layer. After finetuning the logits layer on the target task, all the teacher model's parameters ($\psi$) are fixed, including $\psi_{\text{logit}}$.

\subsection{Input Gradient Matching}
The aim of the input gradient matching is to train the student model to generate input gradients that semantically resemble those from the teacher model. The input gradient characterizes how the loss value is affected by small changes to each input pixel. 

We express the classification cross entropy loss of the student model $f_{\text{s}}$ on the target task dataset $\mathcal{D}_{\text{target}}$ as:
\begin{equation} \label{eq:xent loss}
    L_{\theta,\text{xent}}(\xinput, \mathbf{y}, \theta) = \mathbb{E}_{(\xinput, \mathbf{y})} \left[ - \mathbf{y}^\top \log f_{\text{s}} (\xinput) \right]
\end{equation}

Through gradient backpropagation, the input gradient of the student model $f_{\text{s}}$ is
\begin{equation}
\Js(\xinput) \coloneqq \nabla_{\xinput} L_{\theta,\text{xent}} = \left[\frac{\partial L_{\theta,\text{xent}}}{\partial {\xinput_{1}}} ~~~\cdots~~~ \frac{\partial L_{\theta,\text{xent}}}{\partial {\xinput_{d}}} \right]
\end{equation}

where $d=hwc$.

Correspondingly, the input gradient of the teacher model $f_{\text{t}}$ is
\begin{equation}
\Jt(\xinput) \coloneqq \nabla_{\xinput} L_{\psi,\text{xent}} = \left[\frac{\partial L_{\psi,\text{xent}}}{\partial {\xinput_{1}}} ~~~\cdots~~~ \frac{\partial L_{\psi,\text{xent}}}{\partial {\xinput_{d}}} \right]
\end{equation}

\subsubsection{Adversarial Regularization} 
To achieve the objective of training the student model's input gradient $\Js$ to resemble those from the teacher model $\Jt$, we draw inspiration from GANs, a framework comprising a generator and discriminator model. In our case, we train the $f_{\text{s}}$ to make it hard for the discriminator $f_{\text{disc}}$ to distinguish between $\Jt$ and $\Js$. The discriminator output value $f_{\text{disc}}(\J)$ represents the probability that $\J$ came from the teacher model $f_{\text{t}}$ rather than $f_{\text{s}}$. To train $f_{\text{s}}$ to produce $\Js$ that $f_{\text{disc}}$ perceive as $\Jt$, we employ the following adversarial loss:

\begin{equation}  \label{eq:adv loss}
\begin{aligned}
L_{\text{adv}} 
& = \mathbb{E}_{\Jt} [ \log f_{\text{disc}}(\Jt) ] + \mathbb{E}_{\Js} [\log (1 - f_{\text{disc}}(\Js))] 
\end{aligned}
\end{equation}

Combining this regularization loss with the classification loss function $L_{\text{xent}}$ in Equation~(\ref{eq:xent loss}), we can optimize through stochastic gradient descent (SGD) to approximate the optimal parameters for $f_{\text{s}}$ as follows,

\begin{equation}  \label{eq:xent and adv loss}
\theta^* = \argmin_{\theta} (L_{\theta,\text{xent}} + \lambda_{\text{adv}} L_{\text{adv}})
\end{equation}

where $\lambda_{\text{adv}}$ control how much input gradient adversarial regularization term dominates the training.

In contrast, the discriminator ($f_{\text{disc}}$) learns to correctly distinguish the input gradients by maximizing the adversarial loss term. Parameterizing $f_{\text{disc}}$ with $\phi$, the discriminator is also trained with SGD as such
\begin{equation}
\phi^* = \argmax_{\phi} L_{\text{adv}}
\end{equation}

\subsubsection{Reconstruction Regularization}
Apart from the adversarial loss term, we also employ a term to penalize the $l_2$ difference between the $\Js$ and $\Jt$ generated from the same input image.

\begin{equation}  \label{eq:l2 diff loss}
\begin{aligned}
L_{\text{diff}} = \| \Js - \Jt \|^2_2 
\end{aligned}
\end{equation}

The $L_{\text{diff}}$ term is analogous to the additional reconstruction loss in a VAE-GAN setup \cite{larsen2015autoencoding} where it has shown to improve performance. For each given input image ($\xinput$) in IGAM, there is a corresponding target input gradient $\Jt$ for the student model's $\Js$ to match, allowing us to exploit this instance matching loss ($L_{\text{diff}}$). Adding this term with Equation~\ref{eq:xent and adv loss}, the final training objective of the student model is  
\begin{equation}  \label{eq:total student loss}
\theta^* = \argmin_{\theta} (L_{\theta,\text{xent}} + \lambda_{\text{adv}} L_{\text{adv}} + \lambda_{\text{diff}} L_{\text{diff}})
\end{equation}

where $\lambda_{\text{diff}}$ determines the weight of the $l_2$ penalty term in the training.

Figure~\ref{fig:IGAM architecture} shows a summary of IGAM training phase while Algorithm~\ref{algo:IGAM Training} details the corresponding pseudo-codes.

\begin{figure}[ht]
    \centering
    \includegraphics[width=0.5\linewidth]{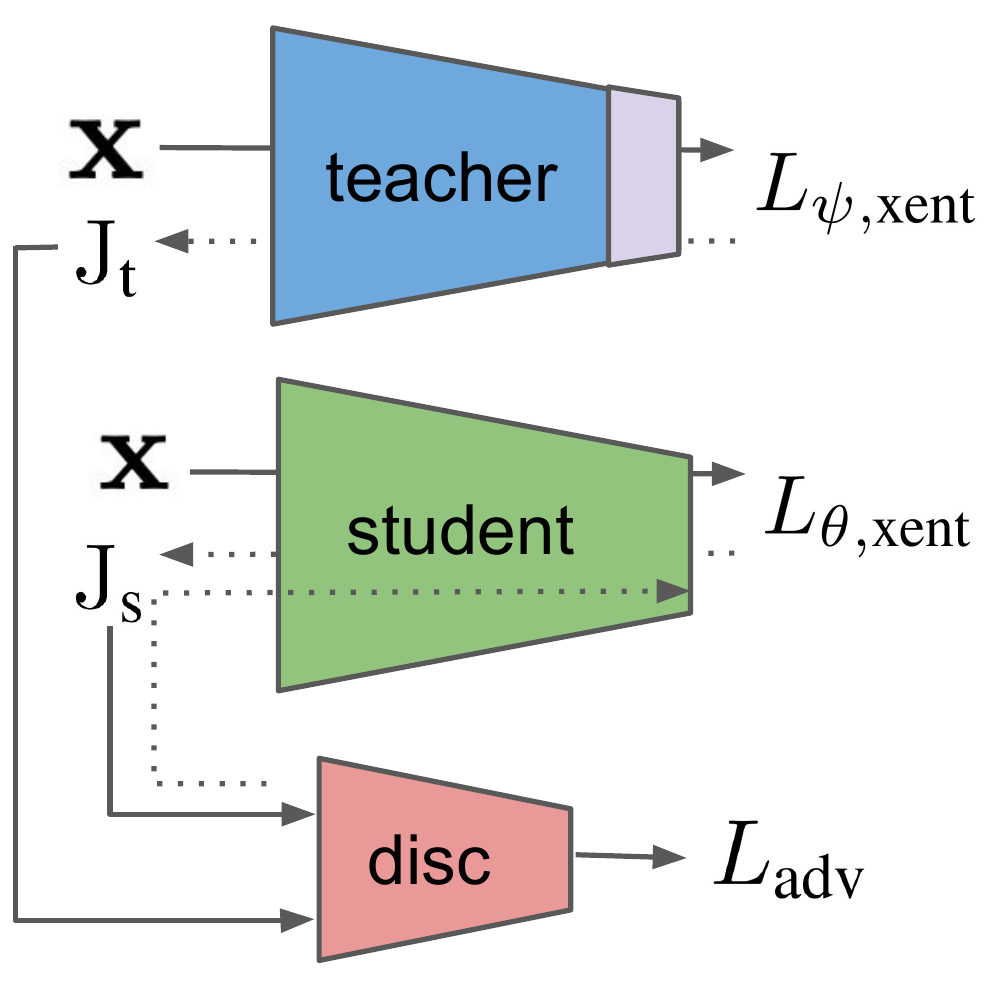}
    \caption{Training phase of input gradient adversarial matching (IGAM).}
    \label{fig:IGAM architecture}
\end{figure}

\begin{algorithm}
\footnotesize
 \caption{Input gradient adversarial matching}
 \label{algo:IGAM Training}

\textbf{Input:} Target task training data $\mathcal{D}_{\text{train}}$, Learning rates for teacher model $f_{\text{t}}$, student model $f_{\text{s}}$ and discriminator $f_{\text{disc}}$: ($\alpha, \beta, \gamma$)

\For{ each finetuning iteration }{
 Sample $(\xinput, \mathbf{y}) \sim \mathcal{D}_{\text{train}}$

 
 $L_{\psi,\text{xent}} \gets - \mathbf{y}^\top \log f_{\text{t}} (\xinput)$ \algorithmiccomment{Classification loss}
 
 $\psi_{\text{logit}} \gets \psi_{\text{logit}} - \alpha ~ \nabla_{\psi_{\text{logit}}} L_{\psi,\text{xent}}$ \algorithmiccomment{Update teacher $f_{\text{t}}$ to minimize $L_{\psi,\text{xent}}$}
 
}

\For{ each training iteration }{
 Sample $(\xinput, \mathbf{y}) \sim \mathcal{D}_{\text{train}}$
 
 $L_{\psi,\text{xent}} \gets - \mathbf{y}^\top \log f_{\text{t}} (\xinput)$ \algorithmiccomment{Classification loss for teacher}
 
 $\Jt \gets \nabla_{\xinput} L_{\psi,\text{xent}}$ \algorithmiccomment{Compute teacher input gradient}
 
 $L_{\theta,\text{xent}} \gets - \mathbf{y}^\top \log f_{\text{s}} (\xinput)$ \algorithmiccomment{Classification loss for student}
 
 $\Js \gets \nabla_{\xinput} L_{\theta,\text{xent}}$ \algorithmiccomment{Compute student input gradient}

 $L_{\text{adv}} \gets \log f_{\text{disc}}(\Jt) + \log ( 1- f_{\text{disc}}(\Js))$ \algorithmiccomment{Adversarial loss}
 
 $L_{\text{diff}} \gets \| \Js - \Jt \|^2_2$ \algorithmiccomment{$l_2$ penalty loss}

 $\theta \gets \theta - \beta ~ \nabla_{\theta} (L_{\theta,\text{xent}} + \lambda_{\text{adv}} L_{\text{adv}} + \lambda_{\text{diff}} L_{\text{diff}})  $ \algorithmiccomment{Update the student $f_{\text{s}}$ to minimize $L_{\theta,\text{xent}}$, $L_{\text{adv}}$ and $L_{\text{diff}}$}
 
$\phi \gets \phi + \gamma ~ \nabla_{\phi} L_{adv}  $  \algorithmiccomment{Update discriminator $f_{\text{disc}}$ to maximize $L_{\text{adv}}$}
 
}
\end{algorithm}

\subsection{Transfer With Different Input Dimensions}
In the earlier sections, we assume that the input dimensions of the teacher and student models are the same. Recall that before finetuning, the teacher model $f_{\text{t}}$ was originally trained on source task samples $( \xinput_{\text{src}}, \mathbf{y}_{\text{src}} ) \sim \mathcal{D}_{\text{src}}, \xinput_{\text{src}} \in \mathbb{R}^{d_{\text{src}}}$ where each $\xinput_{\text{src}}$ is a $h_{\text{src}} \times w_{\text{src}}$-size image with $c_{\text{src}}$ channels. In practice, the image dimensions may differ from those from the task target, i.e., $d_{\text{src}} \neq d_{\text{tar}}$. To allow the gradient backpropagation of the losses through the input gradients, we use affine functions to adapt the target task images to match the dimension of the teacher model's input layer:

\begin{equation} 
\begin{aligned}
\xinput_{\text{tar}}' = \mathbf{A} \cdot \xinput_{\text{tar}} + \mathbf{b}
\end{aligned}
\end{equation}

where $\xinput_{\text{tar}}', \mathbf{b} \in \mathbb{R}^{d_{\text{src}}}, \xinput_{\text{tar}} \in \mathbb{R}^{d_{\text{tar}}}$ and $ \mathbf{A} \in \mathbb{R}^{d_{\text{src}} \times d_{\text{tar}}}$.

Subsequently, cross-entropy loss for the teacher model can be computed:
\begin{equation} \label{eq:teacher xent loss}
    L_{\psi,\text{xent}}(\xinput_{\text{tar}}, \mathbf{y}_{\text{tar}}, \psi) = \mathbb{E}_{(\xinput_{\text{tar}}, \mathbf{y}_{\text{tar}})} \left[ - \mathbf{y}_{\text{tar}}^\top \log f_{\text{t}} (\xinput_{\text{tar}}') \right]
\end{equation}

Since affine functions are continuously differentiable, we can backprop to get the input gradient:
\begin{equation}
\Jt(\xinput_{\text{tar}}) = \nabla_{\xinput_{\text{tar}}} L_{\psi,\text{xent}}
\end{equation}

We use a range of such transformations in our experiments to cater for the difference of input dimensions from various source-target dataset pairs. 

\subsubsection{Input Resizing} \label{sec:input resizing}
Image resizing is one such transformation where the resized image can be expressed as the output of an affine function, i.e., $\xinput_{\text{tar}}' = \mathbf{A} \cdot \xinput_{\text{tar}}$. In the case where the teacher model's input dimension is smaller than the student model, i.e., $d_{\text{tar}} > d_{\text{src}}$, we can use average pooling to downsize the image. A $2\times2$ average pooling is equivalent to resizing with bilinear interpolation when $d_{\text{tar}}$ is a multiple of $d_{\text{src}}$. Figure~\ref{fig:image avgpool} shows how we use input resizing to generate the input gradient from the teacher model. For cases of $d_{\text{tar}} < d_{\text{src}}$, we use image resizing with bilinear interpolation to upscale the input dimension before feeding into the teacher model. For the source-target pair of MNIST-CIFAR, we can similarly reduce the number of channels by averaging the RGB values of the CIFAR images before feeding to the teacher model (trained on MNIST).

\subsubsection{Input Cropping} \label{sec:input cropping}
Cropping is another way to downsize the image to fit a smaller teacher model's input dimension, i.e., $d_{\text{tar}} > d_{\text{src}}$. The cropped image is output of $\xinput_{\text{tar}}' = \mathbf{A} \cdot \xinput_{\text{tar}}$ where $\mathbf{A}$ is a row-truncated identity matrix. For input cropping, the initial $\Jt$ would have zero values at the region where the image was cropped out since those pixel values are multiplied by zero. To prevent the discriminator from exploiting this property to distinguish $\Jt$ from $\Js$, we feed into the discriminator $\Jt$ and $\Js$ that are cropped to size $d_{\text{src}}$. Figure~\ref{fig:image crop} shows how we use cropping to generate the cropped input gradient from the teacher model. 

\subsubsection{Input Padding} \label{sec:input padding}
In contrast to cropping, padding can be used for cases where $d_{\text{tar}} < d_{\text{src}}$. With the same form of affine function $\xinput_{\text{tar}}' = \mathbf{A} \cdot \xinput_{\text{tar}}$, $\mathbf{A}$ is a identity matrix preppended and appended with zero-valued rows. Figure~\ref{fig:image pad} shows how we generate the input gradient from the teacher model with input padding.

\begin{figure}[!htbp]
\begin{subfigure}{0.48\linewidth}
  \centering
  \includegraphics[width=\linewidth]{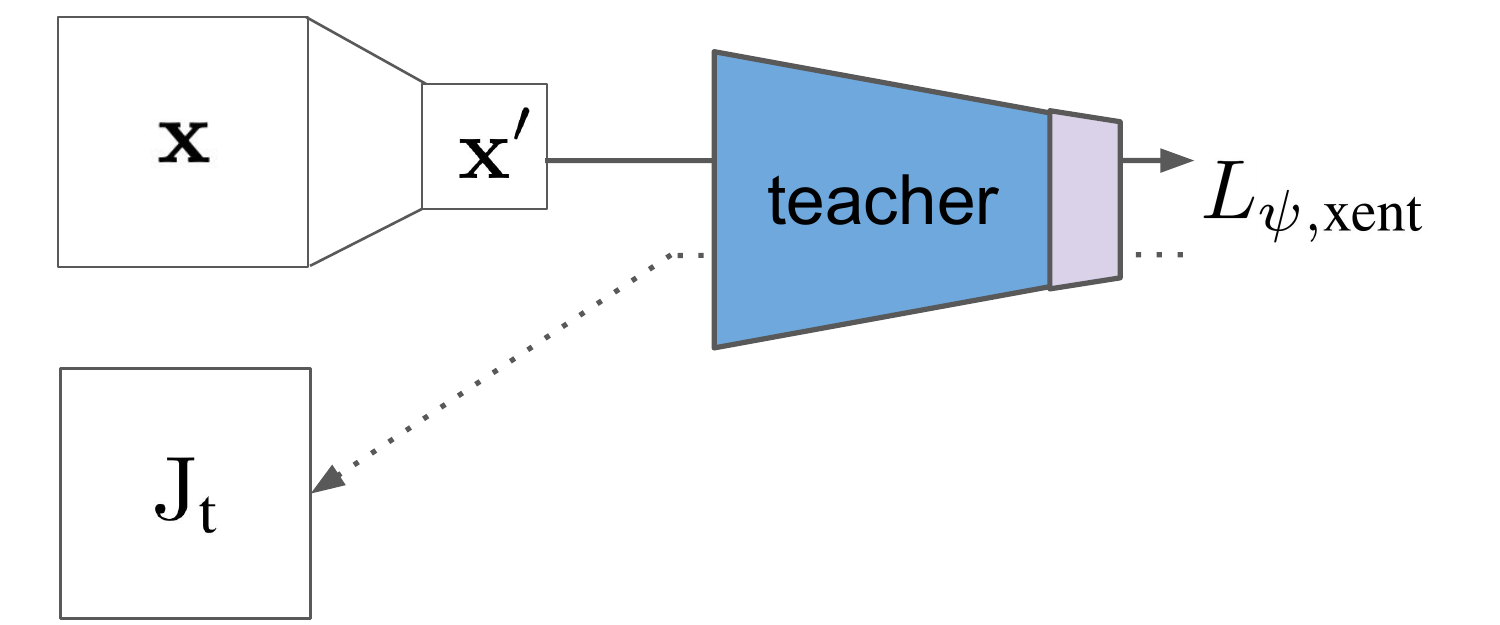}
  \caption{Input resizing}
  \label{fig:image avgpool}
\end{subfigure}
\begin{subfigure}{0.48\linewidth}
  \centering
  \includegraphics[width=\linewidth]{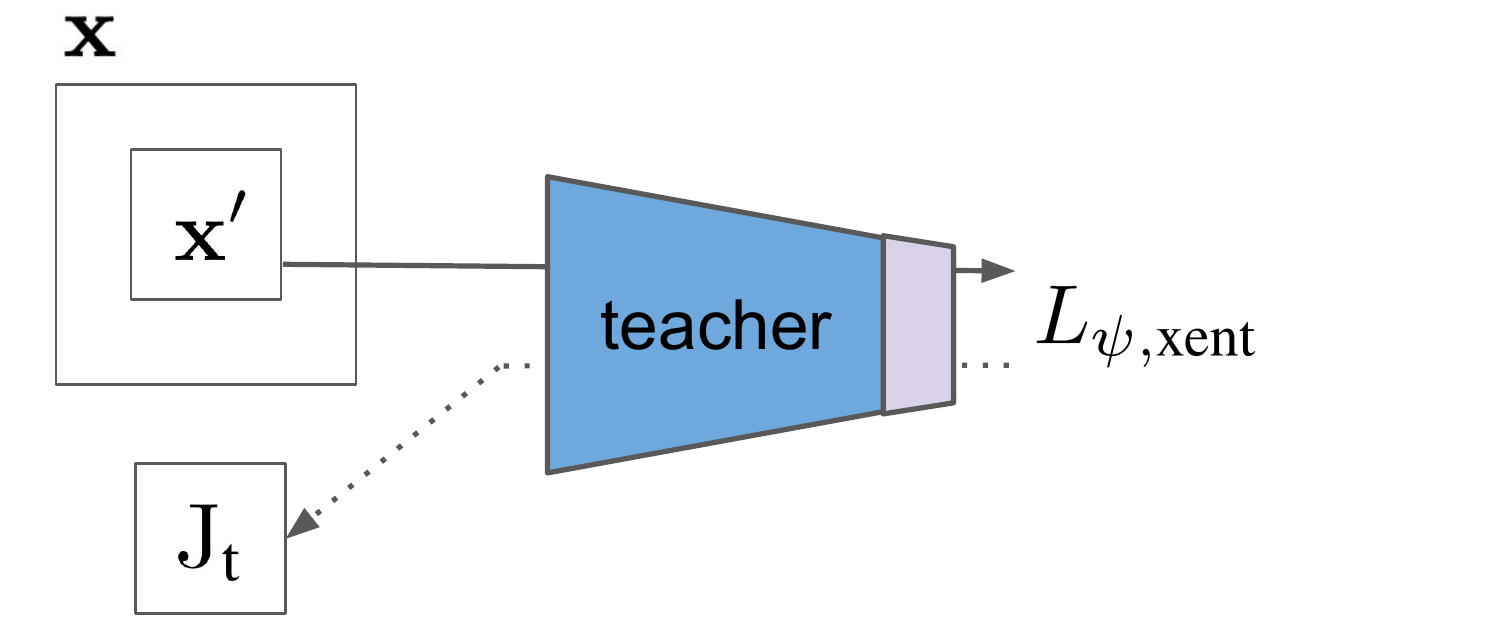}
  \caption{Input cropping}
  \label{fig:image crop}
\end{subfigure}
\begin{subfigure}{\linewidth}
  \centering
  \includegraphics[width=0.48\linewidth]{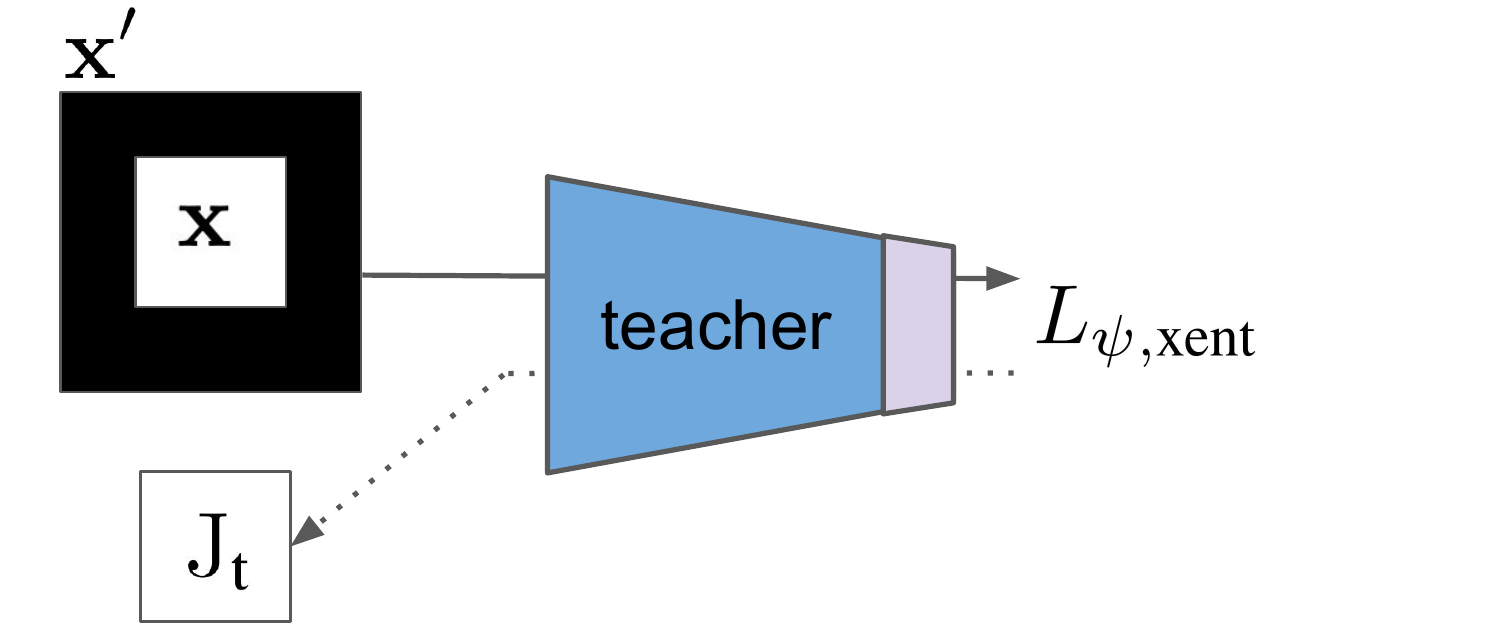}
  \caption{Input padding}
  \label{fig:image pad}
\end{subfigure}
\caption{Transformations to fit images to teacher model's input dimensions.}
\label{fig:input transformations}
\end{figure}

\section{Experiments}
We conducted experiments with IGAM on source-target data pairs comprising of MNIST, CIFAR-10, CIFAR-100 and Tiny-ImageNet. These datasets allow us to validate the effectiveness of IGAM in transferring across tasks with different image dimensions. Unless otherwise stated, adversarial robustness is evaluated based on $l_{\infty}$ adversarial examples with $\varepsilon = \frac{8}{255})$. IGAM's hyperparameters such as $\lambda_{\text{adv}}, \lambda_{\text{diff}}$ and $f_{disc}$ for each experiment are included in the supplementary material.


\subsection{CIFAR-10 Target Task} \label{sec:cifar10 target task}
In our experiments with CIFAR-10 as the target task, we study two types of robustness transfer. The upwards transfer involves employing IGAM to transfer robustness from a smaller model trained on the simpler MNIST dataset to a larger CIFAR-10 classifier. Conversely, the downwards transfer experiments involve transferring robustness from a 200-class Tiny-ImageNet model to a CIFAR-10 classifier.

\subsubsection{Upwards Transfer} \label{sec:cifar10 upwards}

\paragraph{Setup}
CIFAR-10 is a 10-class colored image dataset comprising of 50k training and 10k test images, each of size $32\times32\times3$. For the CIFAR-10 student model, we use a Wide-Resnet 32-10 model with similar hyperparameters to \cite{madry2017towards} and train it for 200 epochs on natural training images with IGAM. The MNIST dataset consists of 60k training and 10k test binary-colored images, each of size $28\times28\times1$. For the robust teacher model trained on MNIST, we also follow the same adversarial training setting and 2-CNN layered architecture as \cite{madry2017towards} \footnote{Robust MNIST pre-trained model downloaded from https://github.com/MadryLab/MNIST\_challenge}. The teacher model is finetuned on natural CIFAR-10 images for 10 epochs before using it to train the student model with IGAM. Since the input dimensions of CIFAR-10 and MNIST are different, we average pool pixel values across the color channels of CIFAR-10 images to get dimension $32\times32\times1$ and subsequently center crop them into $28\times28\times1$ input images for the MNIST teacher model. With this same input transformation, we also finetune the final logit layer of a robust MNIST model on CIFAR-10 images similar to \cite{shafahi2019adversarially} for 100 epochs, to compare as a baseline (FT-MNIST). We also train a strong robust classifier, with 7-step PGD adversarial training like in \cite{madry2017towards}, with the same architecture as the IGAM student model to compare.

\paragraph{Results}
In the face of adversarial examples, the IGAM-trained student model outperforms the standard and finetuned baselines by large margins (Table~\ref{tab:cifar10}). Despite the difference between the dataset domains and model architectures, IGAM can transfer robustness from the teacher to the student model to almost match that from a strong adversarially trained (AT) model. The IGAM student model has higher clean test accuracy than the robust PGD7-trained baseline which we believe is a result of using natural (not adversarially perturbed) images as training data in IGAM.

We note that though finetuning was previously showed to have positive results in transferring robustness across relatively similar domains like between CIFAR10 and CIFAR100 \cite{shafahi2019adversarial}, it fails to transfer successfully here. This is likely due to the bigger difference between the MNIST and CIFAR-10 dataset, as well as the requirement of a more sophisticated model architecture for the more challenging CIFAR-10 dataset.

\begin{table}[ht]
    \centering
    \footnotesize
    \caption{Accuracy (\%) on clean and adversarial CIFAR-10 test samples with upwards transfer.}
        \begin{tabular}{ lccccc }
         \hline
         Model & Clean & FGSM & PGD5 & PGD10 & PGD20 \\
         \hline
         Standard & \textbf{95.0} & 13.4 & 0 & 0 & 0 \\
         FT-MNIST & 33.4 & 1.51 & 0.44 & 0.15 & 0.12 \\
         IGAM-MNIST & 93.6 & \textbf{67.8} & \textbf{63.6} & \textbf{56.9} & \textbf{43.5} \\
         \hline
         PGD7-trained & 87.3 & 56.2 & 55.5 & 47.3 & 45.9 \\
         \hline
        \end{tabular}
\label{tab:cifar10}
\end{table}

\subsubsection{Downwards Transfer}
\paragraph{Setup}
Tiny-ImageNet is a 200-class image dataset where each class contains 500 training and 50 test images. Each Tiny-ImageNet image has dimension of $64\times64\times3$. For the robust teacher model trained on Tiny-ImageNet, we use a similar Wide-Resnet 32-10 model since it is compatible with a larger input dimension due to its global average pooling operation of the feature maps before fully connected layers. We robustly train this teacher model on Tiny-ImageNet, following the same adversarial training hyperparameters in \cite{madry2017towards} where robust models are trained with $l_{\infty}$ adversarial examples generated by 7-step PGD. Before using it to train the student model with IGAM, the teacher model is finetuned on natural CIFAR-10 images for 6 epochs. Since the input dimensions of CIFAR-10 and Tiny-ImageNet are different, we resize the $32\times32\times3$ CIFAR-10 images with bilinear interpolation to get dimension $64\times64\times3$ for finetuning the teacher model. For the IGAM student model, we use the same Wide-Resnet 32-10 model and hyperparameters as in \S~\ref{sec:cifar10 upwards}. We also finetune the final logit layer of a robust Tiny-ImageNet model on upsized CIFAR-10 images similar to \cite{shafahi2019adversarially} for 100 epochs, to compare as a baseline (FT-TinyImagenet). We also investigate two more types of input transformation for IGAM here. The first is a trained $3\times3$ transpose convolutional filter, with stride 2, to upscale the CIFAR-10 images to size $64\times64\times3$. This single transpose convolutional layer is trained together with the teacher model while finetuning on natural CIFAR-10 images. The second type of input transformation is padding, as detailed in \S~\ref{sec:input padding}, of which we explore two variants: center-padding and random-padding.

\paragraph{Results}
With input padding or input resizing, the IGAM-trained student model outperforms the standard and finetuned baselines in adversarial robustness (Table~\ref{tab:cifar10-downwards}). From our experiments, using padding or resizing is more effective for downwards transfer of robustness, with slightly better results for resizing. With the downwards transfer, the student model can match the strong PGD7-trained baseline even more closely than in the upwards transfer case (Table~\ref{tab:cifar10}). This is expected since the teacher model was robustly trained in a more challenging Tiny-ImageNet task and would likely learn even more robust representations than if it were trained on the simpler datasets like MNIST. Compared to upwards transfer, the finetuning baseline transfers robustness and clean accuracy performance to a larger extent but is still outperformed by IGAM.

\begin{table*}[ht]
    \centering
    \scriptsize
    \caption{Accuracy (\%) on clean and adversarial CIFAR-10 test samples with downwards transfer.}
        \begin{tabular}{ lccccccc }
         \hline
         Model & Clean & FGSM & PGD5 & PGD10 & PGD20 & PGD50 & PGD100 \\
         \hline
         Standard & \textbf{95.0} & 13.4 & 0 & 0 & 0 & 0 & 0 \\
         FT-TinyImagenet & 77.2 & 37.7 & 33.9 & 28.0 & 24.9 & 23.0 & 22.5 \\
         IGAM-TransposeConv & 93.2 & \textbf{65.0} & \textbf{58.8} & 44.5 & 32.4 & 22.4 & 18.7 \\
         IGAM-RandomPad & 88.3 & 35.8 & 43.9 & 40.1 & 38.6 & 37.8 & 37.6 \\
         IGAM-Pad & 87.9 & 51.6 & 52.2 & 46.6 & 44.0 & 43.0 & 42.5 \\
         IGAM-Upsize & 88.7 & 54.0 & 52.5 & \textbf{47.6} & \textbf{45.1} & \textbf{43.5} & \textbf{43.0} \\
         \hline
         PGD7-trained & 87.25 & 56.22 & 55.5 & 47.3 & 45.9 & 45.4 & 45.3 \\
         \hline
        \end{tabular}
\label{tab:cifar10-downwards}
\end{table*}

\subsubsection{Input Gradients}
When comparing the input gradients of the various baseline and IGAM models (Figure~\ref{fig:input gradients}), we can observe that there is a diverse degree of saliency. The IGAM models' input gradients appear less noisy than a standard trained model as what we aim to achieve with our proposed method. Interestingly, the IGAM-MNIST model's input gradients have a degree of saliency despite the sparse input gradients from its FT-MNIST teacher model. For IGAM models with a Tiny-ImageNet teacher, the more robust variants like IGAM-Upsize and IGAM-Pad display less noisy input gradients than the less robust IGAM-RandomPad and IGAM-TransposeConv. More input gradient samples are displayed in Figure~\ref{fig:input gradients appendix} of the supplementary material.

\begin{figure*}[ht]
    \centering
    \includegraphics[width=0.5\linewidth]{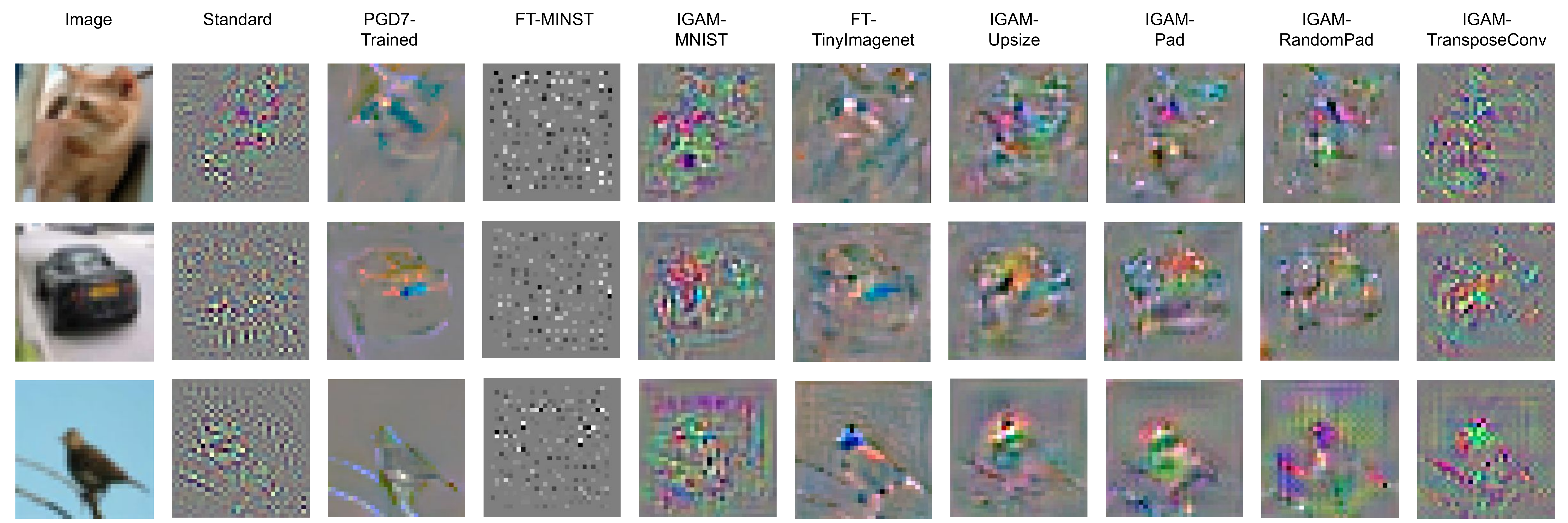}
    \caption{Input gradients of different models.}
    \label{fig:input gradients}
\end{figure*}

\subsection{CIFAR-100 Target Task}
We further study IGAM performance in upwards transfer of robustness with CIFAR-100 as the target task, MNIST and CIFAR-10 as the source task.

\paragraph{Setup}
CIFAR-100 is a 100-class colored image dataset comprising of 50k training and 10k test images. Similar to CIFAR-10, each image has a dimension of $32\times32\times3$. For the CIFAR-100 student model, we use a Wide-Resnet 32-10 model with similar hyperparameters as \S~\ref{sec:cifar10 upwards} except for the final logit layer, which has 100 instead of 10 class outputs. We train the student model for 200 epochs on natural CIFAR-100 training images with IGAM. The robust MNIST teacher model used is similar to the one in \S~\ref{sec:cifar10 upwards}. For the robust CIFAR-10 teacher model, we also follow the same adversarial training setting and architecture as \cite{madry2017towards} \footnote{Robust CIFAR-10 pre-trained model downloaded from https://github.com/MadryLab/cifar10\_challenge}. During IGAM training with MNIST as the source task, the input transformation same as in \S~\ref{sec:cifar10 upwards} is used to resize CIFAR-100 images into $28\times28\times1$ inputs for the teacher model. No input transformation is used when the source task is CIFAR-10 since its images have the same dimensions as CIFAR-100's. The final logit layers of MNIST and CIFAR-10 teacher models are finetuned for 10 and 6 epochs, respectively, on natural CIFAR-100 images before been used to transfer robustness in IGAM. We also finetune the final logit layer of a robust CIFAR-10 model on CIFAR-100 for 100 epochs, to compare as a baseline (FT-CIFAR10). We also train a strong robust classifier, with 7-step PGD adversarial training like in \cite{madry2017towards}, with the same architecture as the IGAM student model to compare.

\paragraph{Results}
Similar to our findings in \S~\ref{sec:cifar10 target task}, IGAM-trained models outperform standard and finetuned baselines in adversarial robustness (Table~\ref{tab:cifar100}). Expectedly, using CIFAR-10 as the source task yields higher transferred robustness than using MNIST for IGAM. Since CIFAR-10 is closer to CIFAR-100 and more challenging than MNIST, the CIFAR-10 teacher model likely has more robust and relevant representations that are reflected as more robust input gradients.

We note that though CIFAR-10 and CIFAR-100 are the most similar datasets in our experiments, the finetuned baseline has lower clean accuracy and adversarial robustness compared to IGAM models. Finetuned models' weights are frozen up until the final logit layer to retain learned robust representations. While weight freezing maintains a degree of robustness to outperform standard training, it may restrict the model from learning new representations relevant to the target task, explaining its lower clean accuracy. We believe this restriction also explains its lower robustness compared to IGAM since IGAM models are free to learn representations important for the target task.

\begin{table}[ht]
    \centering
    \footnotesize
    \caption{Accuracy (\%) on clean and adversarial CIFAR-100 test samples.}
        \begin{tabular}{ lccccc }
         \hline
         Model & Clean & FGSM & PGD5 & PGD10 & PGD20 \\
         \hline
         Standard & \textbf{78.7} & 7.95 & 0.13 & 0.03 & 0\\
         FT-CIFAR10 & 49.3 & 17.2 & 15.3 & 11.7 & 10.5 \\
         IGAM-MNIST & 73.16 & \textbf{41.41} & \textbf{33.09} & 23.35 & 17.67 \\
         IGAM-CIFAR10 & 62.39 & 34.31 & 29.59 & \textbf{24.05} & \textbf{21.74} \\
         \hline
         PGD7-trained & 60.4 & 29.1 & 29.3 & 24.3 & 23.5 \\
         \hline
        \end{tabular}
\label{tab:cifar100}
\end{table}

\paragraph{Roles of Loss Terms}
Improvements from the two terms are additive to each other, as reflected in Table~\ref{tab:lambda diff} and \ref{tab:lambda adv}. From Figure~\ref{fig:all landscapes} in the supplementary material, we observe that both the $L_{\text{adv}}$ and $L_{\text{diff}}$ smoothen the decision boundaries and lower cross-entropy values in the loss landscape compared to the standard trained baseline. 

\begin{table}[ht]
    \centering
    \footnotesize
    \scriptsize
    \caption{IGAM-CIFAR10 accuracy (\%) with varying $\lambda_{\text{diff}}$.}
        \begin{tabular}{ l|cccc }
         \hline
         $\lambda_{\text{diff}}$ & 0 & 2.5 & 5 & 10 \\
         \hline
         PGD20 & 16.0 & 16.3 & 21.7 & 21.7\\
         Clean & 58.9 & 61.8 & 62.9 & 62.4\\
         \hline
        \end{tabular}
\label{tab:lambda diff}
\end{table}

\begin{table}[ht]
    \centering
    \footnotesize
    \scriptsize
    \caption{IGAM-CIFAR10 accuracy (\%) with varying $\lambda_{\text{adv}}$.}
        \begin{tabular}{ l|cccc }
         \hline
         $\lambda_{\text{adv}}$ & 0 & 0.5 & 1 & 2 \\
         \hline
         PGD20 & 3.9 & 4.34 & 7.37 & 21.7\\
         Clean & 78.4 & 77.4 & 74.3 & 62.4\\
         \hline
        \end{tabular}
\label{tab:lambda adv}
\end{table}

\paragraph{Compute Time}
Since finetuning is conducted once, we amortize its time taken over each IGAM epoch to arrive at 347s, which is lower than the 815s taken for a 7-step PGD epoch. Even though IGAM involves an additional discriminator update step on top of standard training, the parameter size of the discriminator is much smaller than the classifier model. 


\subsection{Tiny-ImageNet Target Task}
We study if robustness can transfer through the input gradients when the target task has significantly larger input dimensions than the source task, with Tiny-ImageNet as the target task and CIFAR-10/100 as the source task.

\paragraph{Setup}
For the robust CIFAR-10/100 teacher model, we follow the same adversarial training setting and Wide-Resnet 32-10 architecture as \cite{madry2017towards}. We use a similar Wide-Resnet 32-10 model for the Tiny-ImageNet student model due to its compatible with larger input dimension due to its global average pooling layer. The robust CIFAR-10/100 teacher models are finetuned for 5 epochs on natural Tiny-ImageNet images before being used for IGAM. Since the input dimensions of Tiny-ImageNet and CIFAR-10/100 are different, we study two types of input transformation to reshape the image dimension from $64\times64\times3$ to $32\times32\times3$ for finetuning the teacher model. The first is image resizing with bilinear interpolation (\S~\ref{sec:input resizing}), which is equivalent to a $2\times2$ average pooling layer with stride 2. The second transformation is center-cropping as detailed in \S~\ref{sec:input cropping}. The models' adversarial robustness is evaluated based on 5-step PGD attacks on test Tiny-ImageNet samples.

\paragraph{Results}
Similar to previous target-source task pairs, IGAM can transfer robustness even to much more challenging dataset, to a degree to outperform the standard trained and finetuned baselines (Figure~\ref{fig:tinyimagenet acc}). There is no visible difference in robustness transferred when using image resizing or center-cropping as the input transformation.


\begin{figure}[!htbp]
    \centering
    \includegraphics[width=0.8\linewidth]{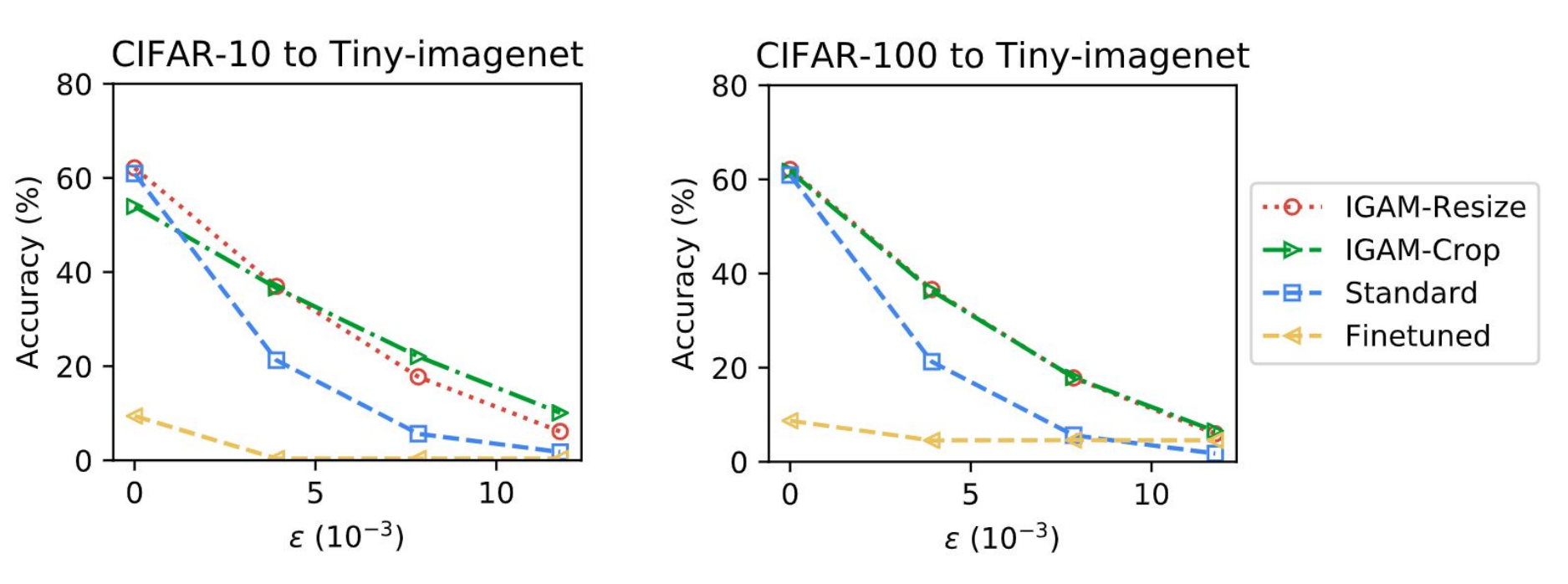}
    \caption{Accuracy (\%) on clean and adversarial Tiny-ImageNet test samples.}
    \label{fig:tinyimagenet acc}
\end{figure}

\section{Theoretical Discussion}
To understand how robustness transfer across input gradients of the student and teacher models, we first look at the link between robustness and saliency of input gradients in a single network. The link is formalized in Theorem 2 of \cite{etmann2019connection} which states that a network's linearized robustness ($\rho$) around an input $\xinput$ is upper bounded by alignment term $\alpha$:
\begin{equation} \label{eq:linearized robustness bound}
\rho(\xinput) \leq \alpha(\xinput) + \frac{C}{\| g \|}
\end{equation}
where $g$ is the Jacobian of the difference between the top two logits, $\alpha(\xinput) = \frac{|\langle \xinput, g \rangle|}{\| g \|}$ and $C$ is a positive constant. An important notion here is that a model with high linearized robustness ($\rho$) retains its original prediction in face of large perturbation but may still perform poorly on clean test data with incorrect original outputs, such as finetuned teachers.

Different finetuned teacher models (FT-MINST and FT-TinyImagnet) display visually different input gradients which we speculate to be a result of being `locked' into their dataset-specific robust features. Different from natural images which have smooth pixel value distributions, MNIST pixels take extreme binary values. From the robustness-alignment link, one can expect the input gradient to also take extreme values, explaining the sparse $J$ of FT-MINST.

With Theorem~\ref{theorem:generator identity mapping} below, IGAM's $L_{\text{adv}}$ term encourages the teacher and student models' input gradients and, consequently, their input alignment terms ($\alpha$) to match well.
\begin{theorem} \label{theorem:generator identity mapping}
The global minimum of $L_{\text{adv}}$ is achieved when $\Js = \Jt$.
\end{theorem}

Its proof is in the supplementary material (\S~\ref{sec:proof appendix}). As a result, the high linearized robustness upper bound of teacher model is transferred to the student model. Though input gradients are approximations of $g$ and the upper bound is not tight, we observe that such transfer is feasible in our experiments. On top of this transferred robustness bound, all of the student model's weights are free to learn features relevant to the target task in boosting its clean accuracy, hence the improved performance over its teacher models. 

\section{Conclusions}
We showed that input gradients are an effective medium to transfer adversarial robustness across different tasks and even across different model architectures. To train a student model's input gradients to semantically match those of a robust teacher model, we proposed input gradient adversarial matching (IGAM) to optimize for the input gradients' source to be indistinguishable for a discriminator network. Through extensive experiments on image classification, IGAM models outperform standard trained models and models finetuned on pre-trained robust feature extractors. This demonstrates that input gradients are a more versatile and effective medium of robustness transfer. We hope that this will encourage new defenses that also target the semantics of input gradients to achieve adversarial robustness.

\subsubsection*{Acknowledgments}
This paper is supported in part by the National Research Foundation, Singapore under its AI Singapore Programme (AISG Award No: AISG-RP-2018-004), and the Data Science and Artificial Intelligence Research Center
at Nanyang Technological University. Any opinions, findings
and conclusions or recommendations expressed in this
material are those of the authors and do not reflect the views
of National Research Foundation, Singapore.

{\small
\bibliographystyle{ieee_fullname}
\bibliography{egbib}

\begin{thebibliography}{10}\itemsep=-1pt

\bibitem{andriushchenko2019provably}
Maksym Andriushchenko and Matthias Hein.
\newblock Provably robust boosted decision stumps and trees against adversarial
  attacks.
\newblock {\em arXiv preprint arXiv:1906.03526}, 2019.

\bibitem{carlini2017towards}
Nicholas Carlini and David Wagner.
\newblock Towards evaluating the robustness of neural networks.
\newblock In {\em 2017 IEEE Symposium on Security and Privacy (SP)}, pages
  39--57. IEEE, 2017.

\bibitem{chan2019jacobian}
Alvin Chan, Yi Tay, Yew~Soon Ong, and Jie Fu.
\newblock Jacobian adversarially regularized networks for robustness.
\newblock {\em arXiv preprint arXiv:1912.10185}, 2019.

\bibitem{croce2019minimally}
Francesco Croce and Matthias Hein.
\newblock Minimally distorted adversarial examples with a fast adaptive
  boundary attack.
\newblock {\em arXiv preprint arXiv:1907.02044}, 2019.

\bibitem{drucker1991double}
Harris Drucker and Yann Le~Cun.
\newblock Double backpropagation increasing generalization performance.
\newblock In {\em IJCNN-91-Seattle International Joint Conference on Neural
  Networks}, volume~2, pages 145--150. IEEE, 1991.

\bibitem{etmann2019connection}
Christian Etmann, Sebastian Lunz, Peter Maass, and Carola-Bibiane
  Sch{\"o}nlieb.
\newblock On the connection between adversarial robustness and saliency map
  interpretability.
\newblock {\em arXiv preprint arXiv:1905.04172}, 2019.

\bibitem{goodfellow2014generative}
Ian Goodfellow, Jean Pouget-Abadie, Mehdi Mirza, Bing Xu, David Warde-Farley,
  Sherjil Ozair, Aaron Courville, and Yoshua Bengio.
\newblock Generative adversarial nets.
\newblock In {\em Advances in neural information processing systems}, pages
  2672--2680, 2014.

\bibitem{goodfellow2014explaining}
Ian~J Goodfellow, Jonathon Shlens, and Christian Szegedy.
\newblock Explaining and harnessing adversarial examples.
\newblock {\em arXiv preprint arXiv:1412.6572}, 2014.

\bibitem{gowal2018effectiveness}
Sven Gowal, Krishnamurthy Dvijotham, Robert Stanforth, Rudy Bunel, Chongli Qin,
  Jonathan Uesato, Timothy Mann, and Pushmeet Kohli.
\newblock On the effectiveness of interval bound propagation for training
  verifiably robust models.
\newblock {\em arXiv preprint arXiv:1810.12715}, 2018.

\bibitem{hein2017formal}
Matthias Hein and Maksym Andriushchenko.
\newblock Formal guarantees on the robustness of a classifier against
  adversarial manipulation.
\newblock In {\em Advances in Neural Information Processing Systems}, pages
  2266--2276, 2017.

\bibitem{hendrycks2019using}
Dan Hendrycks, Kimin Lee, and Mantas Mazeika.
\newblock Using pre-training can improve model robustness and uncertainty.
\newblock {\em arXiv preprint arXiv:1901.09960}, 2019.

\bibitem{jakubovitz2018improving}
Daniel Jakubovitz and Raja Giryes.
\newblock Improving dnn robustness to adversarial attacks using jacobian
  regularization.
\newblock In {\em Proceedings of the European Conference on Computer Vision
  (ECCV)}, pages 514--529, 2018.

\bibitem{kannan2018adversarial}
Harini Kannan, Alexey Kurakin, and Ian Goodfellow.
\newblock Adversarial logit pairing.
\newblock {\em arXiv preprint arXiv:1803.06373}, 2018.

\bibitem{larsen2015autoencoding}
Anders Boesen~Lindbo Larsen, S{\o}ren~Kaae S{\o}nderby, Hugo Larochelle, and
  Ole Winther.
\newblock Autoencoding beyond pixels using a learned similarity metric.
\newblock {\em arXiv preprint arXiv:1512.09300}, 2015.

\bibitem{lecun2015deep}
Yann LeCun, Yoshua Bengio, and Geoffrey Hinton.
\newblock Deep learning.
\newblock {\em nature}, 521(7553):436, 2015.

\bibitem{liao2018defense}
Fangzhou Liao, Ming Liang, Yinpeng Dong, Tianyu Pang, Xiaolin Hu, and Jun Zhu.
\newblock Defense against adversarial attacks using high-level representation
  guided denoiser.
\newblock In {\em Proceedings of the IEEE Conference on Computer Vision and
  Pattern Recognition}, pages 1778--1787, 2018.

\bibitem{lin2019coco}
Chieh~Hubert Lin, Chia-Che Chang, Yu-Sheng Chen, Da-Cheng Juan, Wei Wei, and
  Hwann-Tzong Chen.
\newblock Coco-gan: Generation by parts via conditional coordinating.
\newblock {\em arXiv preprint arXiv:1904.00284}, 2019.

\bibitem{madry2017towards}
Aleksander Madry, Aleksandar Makelov, Ludwig Schmidt, Dimitris Tsipras, and
  Adrian Vladu.
\newblock Towards deep learning models resistant to adversarial attacks.
\newblock {\em arXiv preprint arXiv:1706.06083}, 2017.

\bibitem{papernot2018cleverhans}
Nicolas Papernot, Fartash Faghri, Nicholas Carlini, Ian Goodfellow, Reuben
  Feinman, Alexey Kurakin, Cihang Xie, Yash Sharma, Tom Brown, Aurko Roy,
  Alexander Matyasko, Vahid Behzadan, Karen Hambardzumyan, Zhishuai Zhang,
  Yi-Lin Juang, Zhi Li, Ryan Sheatsley, Abhibhav Garg, Jonathan Uesato, Willi
  Gierke, Yinpeng Dong, David Berthelot, Paul Hendricks, Jonas Rauber, and
  Rujun Long.
\newblock Technical report on the cleverhans v2.1.0 adversarial examples
  library.
\newblock {\em arXiv preprint arXiv:1610.00768}, 2018.

\bibitem{prakash2018deflecting}
Aaditya Prakash, Nick Moran, Solomon Garber, Antonella DiLillo, and James
  Storer.
\newblock Deflecting adversarial attacks with pixel deflection.
\newblock In {\em Proceedings of the IEEE conference on computer vision and
  pattern recognition}, pages 8571--8580, 2018.

\bibitem{qin2019adversarial}
Chongli Qin, James Martens, Sven Gowal, Dilip Krishnan, Alhussein Fawzi, Soham
  De, Robert Stanforth, Pushmeet Kohli, et~al.
\newblock Adversarial robustness through local linearization.
\newblock {\em arXiv preprint arXiv:1907.02610}, 2019.

\bibitem{raghunathan2018semidefinite}
Aditi Raghunathan, Jacob Steinhardt, and Percy~S Liang.
\newblock Semidefinite relaxations for certifying robustness to adversarial
  examples.
\newblock In {\em Advances in Neural Information Processing Systems}, pages
  10877--10887, 2018.

\bibitem{ross2018improving}
Andrew~Slavin Ross and Finale Doshi-Velez.
\newblock Improving the adversarial robustness and interpretability of deep
  neural networks by regularizing their input gradients.
\newblock In {\em Thirty-second AAAI conference on artificial intelligence},
  2018.

\bibitem{schott2018towards}
Lukas Schott, Jonas Rauber, Matthias Bethge, and Wieland Brendel.
\newblock Towards the first adversarially robust neural network model on mnist.
\newblock {\em arXiv preprint arXiv:1805.09190}, 2018.

\bibitem{shafahi2019adversarial}
Ali Shafahi, Mahyar Najibi, Amin Ghiasi, Zheng Xu, John Dickerson, Christoph
  Studer, Larry~S Davis, Gavin Taylor, and Tom Goldstein.
\newblock Adversarial training for free!
\newblock {\em arXiv preprint arXiv:1904.12843}, 2019.

\bibitem{shafahi2019adversarially}
Ali Shafahi, Parsa Saadatpanah, Chen Zhu, Amin Ghiasi, Christoph Studer,
  David~W. Jacobs, and Tom Goldstein.
\newblock Adversarially robust transfer learning.
\newblock {\em CoRR}, abs/1905.08232, 2019.

\bibitem{szegedy2013intriguing}
Christian Szegedy, Wojciech Zaremba, Ilya Sutskever, Joan Bruna, Dumitru Erhan,
  Ian Goodfellow, and Rob Fergus.
\newblock Intriguing properties of neural networks.
\newblock {\em arXiv preprint arXiv:1312.6199}, 2013.

\bibitem{touvron2019fixing}
Hugo Touvron, Andrea Vedaldi, Matthijs Douze, and Herv{\'e} J{\'e}gou.
\newblock Fixing the train-test resolution discrepancy.
\newblock {\em arXiv preprint arXiv:1906.06423}, 2019.

\bibitem{tsipras2018robustness}
Dimitris Tsipras, Shibani Santurkar, Logan Engstrom, Alexander Turner, and
  Aleksander Madry.
\newblock Robustness may be at odds with accuracy.
\newblock {\em arXiv preprint arXiv:1805.12152}, 2018.

\bibitem{wong2018scaling}
Eric Wong, Frank Schmidt, Jan~Hendrik Metzen, and J~Zico Kolter.
\newblock Scaling provable adversarial defenses.
\newblock In {\em Advances in Neural Information Processing Systems}, pages
  8400--8409, 2018.

\bibitem{xie2019feature}
Cihang Xie, Yuxin Wu, Laurens van~der Maaten, Alan~L Yuille, and Kaiming He.
\newblock Feature denoising for improving adversarial robustness.
\newblock In {\em Proceedings of the IEEE Conference on Computer Vision and
  Pattern Recognition}, pages 501--509, 2019.

\bibitem{zhang2019defense}
Haichao Zhang and Jianyu Wang.
\newblock Defense against adversarial attacks using feature scattering-based
  adversarial training.
\newblock {\em arXiv preprint arXiv:1907.10764}, 2019.

\bibitem{zhang2019theoretically}
Hongyang Zhang, Yaodong Yu, Jiantao Jiao, Eric~P Xing, Laurent~El Ghaoui, and
  Michael~I Jordan.
\newblock Theoretically principled trade-off between robustness and accuracy.
\newblock {\em arXiv preprint arXiv:1901.08573}, 2019.

\end{thebibliography}
}

\appendix

\clearpage


\begin{figure*}[!htbp]
    \centering
    \includegraphics[width=0.8\linewidth]{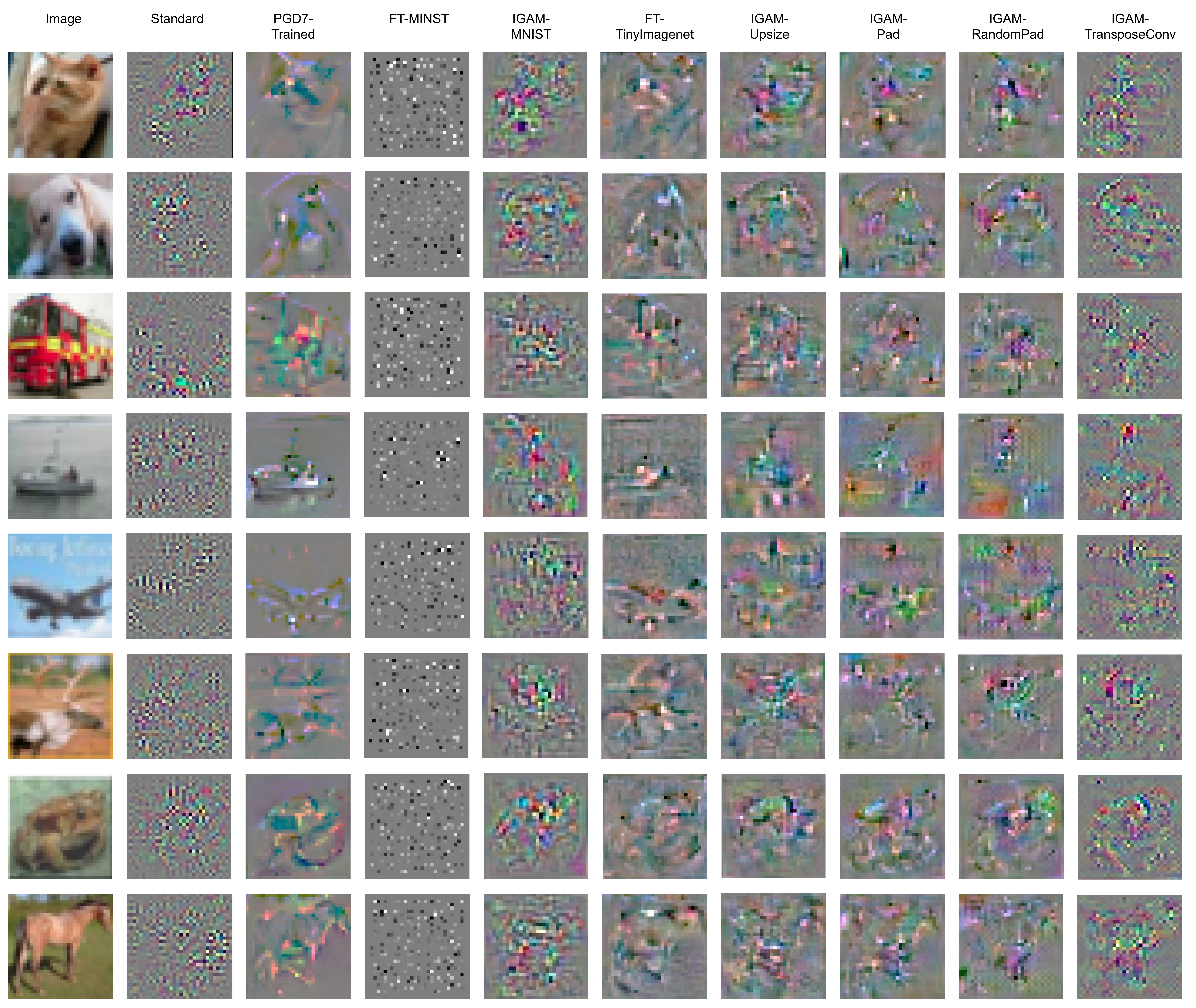}
    \caption{Input gradients of different models.}
    \label{fig:input gradients appendix}
\end{figure*}

\begin{figure*}[!htbp]
\centering
\includegraphics[width=\linewidth]{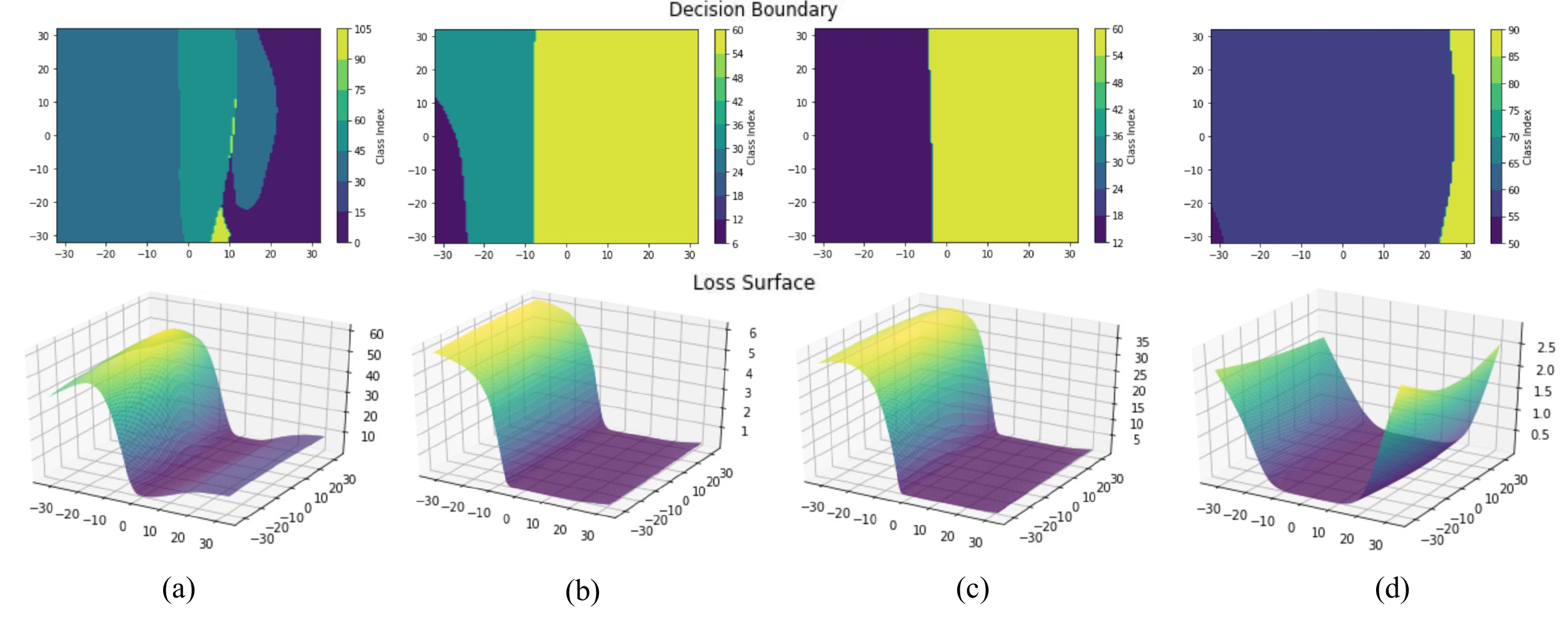}
\caption{Decision boundaries and loss landscapes of (a) standard trained, (b) IGAM-CIFAR10 ($\lambda_{\text{adv}} = 2, \lambda_{\text{diff}} = 0$), (c) IGAM-CIFAR10 ($\lambda_{\text{adv}} = 0, \lambda_{\text{diff}} = 10$) and (d) IGAM-CIFAR10 ($\lambda_{\text{adv}} = 2, \lambda_{\text{diff}} = 10$) along the adversarial perturbation and a random direction. Correct class: \#53.}
\label{fig:all landscapes}
\end{figure*}

\section{IGAM Hyperparameters}
The IGAM hyperparameters are fined through grid search through the same range of hyperparameter values within each transfer task. We report the values of the IGAM models whose results are reported in this paper for reproducibility.

\subsection{CIFAR-10 Target Task}
\paragraph{IGAM-MNIST} 
$\lambda_{\text{adv}} = 1$,
$\lambda_{\text{diff}} = 100$,
\linebreak
$f_{disc}$ : 5 CNN layers (16-32-64-128-256 output channels) and updated once for every 10 classifier update steps

\paragraph{IGAM-TranposeConv} 
$\lambda_{\text{adv}} = 1$,
$\lambda_{\text{diff}} = 10$,
\linebreak
$f_{disc}$ : 4 CNN layers (8-16-32-64 output channels) and updated once for every 5 classifier update steps

\paragraph{IGAM-RandomPad} 
$\lambda_{\text{adv}} = 1$,
$\lambda_{\text{diff}} = 10$,
\linebreak
$f_{disc}$ : 4 CNN layers (8-16-32-64 output channels) and updated once for every 5 classifier update steps

\paragraph{IGAM-Pad} 
$\lambda_{\text{adv}} = 2$,
$\lambda_{\text{diff}} = 20$,
\linebreak
$f_{disc}$ : 4 CNN layers (8-16-32-64 output channels) and updated once for every 5 classifier update steps

\paragraph{IGAM-Upsize} 
$\lambda_{\text{adv}} = 5$,
$\lambda_{\text{diff}} = 10$,
\linebreak
$f_{disc}$ : 4 CNN layers (8-16-32-64 output channels) and updated once for every 5 classifier update steps

\subsection{CIFAR-100 Target Task}
\paragraph{IGAM-MNIST} 
$\lambda_{\text{adv}} = 0.1$,
$\lambda_{\text{diff}} = 200$,
\linebreak
$f_{disc}$ : 5 CNN layers (16-32-64-128-256 output channels) and updated once for every 5 classifier update steps

\paragraph{IGAM-CIFAR10} 
$\lambda_{\text{adv}} = 2$,
$\lambda_{\text{diff}} = 10$,
\linebreak
$f_{disc}$ : 5 CNN layers (16-32-64-128-256 output channels) and updated once for every 10 classifier update steps

\subsection{Tiny-ImageNet Target Task}
\paragraph{IGAM-CIFAR10-Resize} 
$\lambda_{\text{adv}} = 0.1$,
$\lambda_{\text{diff}} = 200$,
\linebreak
$f_{disc}$ : 4 CNN layers (8-16-32-64 output channels) and updated once for every 5 classifier update steps

\paragraph{IGAM-CIFAR10-Crop} 
$\lambda_{\text{adv}} = 2$,
$\lambda_{\text{diff}} = 50$,
\linebreak
$f_{disc}$ : 4 CNN layers (8-16-32-64 output channels) and updated once for every 5 classifier update steps

\paragraph{IGAM-CIFAR100-Resize} 
$\lambda_{\text{adv}} = 0.1$,
$\lambda_{\text{diff}} = 200$,
\linebreak
$f_{disc}$ : 4 CNN layers (8-16-32-64 output channels) and updated once for every 5 classifier update steps

\paragraph{IGAM-CIFAR100-Crop} 
$\lambda_{\text{adv}} = 0.5$,
$\lambda_{\text{diff}} = 200$,
\linebreak
$f_{disc}$ : 4 CNN layers (8-16-32-64 output channels) and updated once for every 5 classifier update steps

\section{Proof} \label{sec:proof appendix}
\begin{theorem} \label{theorem:generator identity mapping appendix}
The global minimum of $L_{\text{adv}}$ is achieved when $\Js = \Jt$.
\end{theorem}

\begin{proof}
From \cite{goodfellow2014generative}, the optimal discriminator is
\begin{equation}
f_{\text{disc}}^* (\J) =  \frac{p_{\text{teacher}} (\J)}{p_{\text{teacher}} (\J) + p_{\text{student}} (\J)}
\end{equation}

We can include the optimal discriminator into Equation~(\ref{eq:adv loss}) to get
\begin{equation}
    \begin{aligned}
L_{\text{adv}}
& = \mathbb{E}_{\J \sim p_{\text{teacher}}} [ \log f_{\text{disc}}^*(\J) ] + \mathbb{E}_{\J \sim p_{\text{student}}} [\log (1 - f_{\text{disc}}^*(\J) )] \\
& = \mathbb{E}_{\J \sim p_{\text{teacher}}} \left[ \log \frac{p_{\text{teacher}} (\J)}{p_{\text{teacher}} (\J) + p_{\text{student}} (\J)} \right] \\
& \phantom{=}\,\phantom{=}\, + \mathbb{E}_{\J \sim p_{\text{student}}} \left[ \log \frac{p_{\text{student}} (\J)}{p_{\text{teacher}} (\J) + p_{\text{student}} (\J)} \right] \\
& = \klinfdiv*{p_{\text{teacher}}}{\frac{p_{\text{teacher}}  + p_{\text{student}} }{2}} \\ 
& \phantom{=}\,\phantom{=}\, + \klinfdiv*{p_{\text{student}}}{\frac{p_{\text{teacher}}+ p_{\text{student}} }{2}} - \log 4\\
& = 2 \cdot JS ( p_{\text{teacher}} || p_{\text{student}} ) - \log 4 \\
    \end{aligned}
\end{equation}

where $KL$ and $JS$ are the Kullback-Leibler and Jensen-Shannon divergence respectively. Since the Jensen-Shannon divergence is always non-negative, $L_{\text{adv}} (G)$ reaches its global minimum value of $- \log 4$ when $JS ( p_{\text{teacher}} || p_{\text{student}} ) = 0$. When $\Js = \Jt$, we get $p_{\text{teacher}} = p_{\text{student}}$ and consequently $JS ( p_{\text{teacher}} || p_{\text{student}} ) = 0$, thus completing the proof.

\end{proof}


\end{document}